\documentclass{article}

\PassOptionsToPackage{numbers,compress}{natbib}
\usepackage[dblblindworkshop, final]{neurips_2026}
\workshoptitle{Long-Context Foundation Models}

\usepackage[utf8]{inputenc}
\usepackage[T1]{fontenc}
\usepackage{hyperref}
\usepackage{url}
\usepackage{booktabs}
\usepackage{amsmath}
\usepackage{amssymb}
\usepackage{amsfonts}
\usepackage{amsthm}
\usepackage{nicefrac}
\usepackage{microtype}
\usepackage{xcolor}

\newtheorem{observation}{Observation}
\DeclareMathOperator*{\argmax}{arg\,max}

\title{Retrieval as Traversal:\\
Value-Indexed Memory for Stateful Long-Context Agents}

\author{
  Daniel Strauß \\
  Technical University of Munich
  \And
  Zhenshan Bing \\
  Nanjing University
  \And
  Alois Knoll \\
  Technical University of Munich
}
\begin{document}

\maketitle

\begin{abstract}
Long-context models repeatedly solve closely related retrieval problems.
Attention or global retrieval can retain large stores, but each forward pass
again searches for what matters. Recurrent compression avoids this search only
by deciding earlier what information to preserve. Neither mechanism directly
exploits that useful memories at neighbouring retrieval steps may substantially
overlap. We propose an alternative in which retrieval is a persistent local
traversal through an address space. We separate a memory's content from
its address. Utility is defined by the information revealed by retrieving
the content, while geometry determines only how cheaply that information can be
reached. This leads to a cost-regularised indexing objective that preserves
information value inside local neighbourhoods while charging for every memory
made locally accessible. Its marginal-value form gives both attractive and
exclusion pressures. Thus, complementary memories should be co-retrievable when the
additional value exceeds retrieval cost, whereas redundant memories should not
occupy the same retrieval set solely because they are individually useful.
As a proof of principle, we test only the simplest fixed-geometry instance on
MiniGrid Memory.
A persistent local traversal head reaches $0.95$ mean return after $0.59$M
training steps with a five-token context, compared with $0.94$M for Gaussian
absolute addressing and more than $2$M for categorical absolute addressing.
The experiment does not test learned geometry or repeated retrieval. Instead, it
motivates the hypothesis that stateful retrieval may be easier to learn than stateless. To test the full proposed method further experimentation is necessary.

\end{abstract}

\section{Introduction}

Long-context models face a basic trade-off. Keeping a large history available
preserves information, but attention or retrieval must repeatedly determine
which parts of that history matter. Compressing the past into a bounded
recurrent state avoids the scan, but the compression is performed before all
future information needs are known. Sparse attention, approximate
nearest-neighbour search, and external retrieval reduce the cost of a query,
but typically still treat successive queries as fresh matching problems
\citep{kitaev2020,malkov2018,roy2021,lewis2020}.

This misses repetition, a third source of cost. The memories useful at
two nearby inference steps need not be independent sets. In order to avoid this cost information about value of the to be retrieved memories has to be known before retrieving them. We investigate this idea as retrieval by traversal. Memories and the current context and retrieval heads are placed in an address space. 
The positions of the current context and the retrieval heads aim to hold the information about the current expected value of the to be retrieved memories, as all memories in $\epsilon$ distance will be retrieved.

Our proposal differs from semantic retrieval in one important respect:
addresses should be shaped by the decision value of retrieving memory
contents, rather than by similarity between the contents themselves. This
distinction matters because two semantically similar memories may be redundant,
whereas two dissimilar memories may only be useful together.

This  paper makes two deliberately limited contributions. First, we give
a decision-theoretic objective for learning such a geometry and derive how
redundancy and complementarity affect local co-retrieval. Second, we report a
small proof-of-principle experiment in which the geometry is fixed and only the
persistent retrieval mechanism is learned. We do not claim that the
learned value-index proposed here has yet been experimentally validated.

\paragraph{Related work.}
External differentiable memories have long separated computation from stored
state.   
Neural Turing Machines combine content and location addressing \citep{graves2014}. Differentiable Neural Computers add dynamic allocation and temporal links \citep{graves2016}.
Lie-access
Neural Turing Machines are particularly close in mechanism. Here, memories occupy a
continuous address manifold and read heads move by relative transformations
\citep{yang2016}. Neural Map instead uses physical position as an external
memory index \citep{parisotto2017}. Modern sparse-attention and retrieval
systems use hashing, learned routing, nearest-neighbor search, or retrieval
indices to avoid attending to every stored item
\citep{kitaev2020,roy2021,wu2022,lewis2020}. We ask instead
whether the address space itself can be organized by the expected utility of
future retrieval and then traversed persistently across related queries.

\section{Value-indexed memory}

\subsection{Content, context, and address}

The content of a memory must be distinguished from the location at which it is
stored. Let
\[
    \Omega_t = \{1,\ldots,n_t\}
\]
index the external memories available at time $t$. Memory $j$ contains a
random variable
\[
    M_j \in \mathcal M .
\]
Before retrieval, $M_j$ is unobserved from the model's current information
state and retrieving the memory reveals a realized content $m_j$.

Let $C_j^{\mathrm w}$ denote the context in which memory $j$ was written. Its
address is
\begin{equation}
    z_j = \Phi(C_j^{\mathrm w}) \in \mathbb R^R .
    \label{eq:address}
\end{equation}
Conditional on a realized memory store, $z_j$ is a fixed location. Retrieving
$j$ inserts $m_j$, not $C_j^{\mathrm w}$, into the model's working
context.

The current context $C_t$, consisting of recent observations together with recently
retrieved memory, contains the sequence of information currently
available to the predictor. The position of the retrieval head is maintained
separately and is not part of $C_t$.

For a realized current context $c$ and a set $S \subseteq \Omega_t$, let
$M_S$ denote the corresponding memory contents in a fixed canonical order.
We define their pre-query information value by
\begin{equation}
    V(S;c)
    =
    \mathbb E\!\left[
        U(c \oplus M_S)-U(c)
        \,\middle|\, C_t=c
    \right],
    \label{eq:value}
\end{equation}
where $U$ is task utility, such as expected return or negative predictive
loss, and $\oplus$ denotes incorporation into the working context. $V$ may also be learned from the next token predictor, through counterfactual training.

Importantly, $V$ is defined from memory contents and no address appears in
(\ref{eq:value}). Geometry affects which contents are cheaply retrievable, but not
what those contents are worth.

For the theoretical discussion we assume free disposal, which is that additional information can not reduce the expected utility. Hence
\begin{equation}
    S\subseteq T
    \quad\Longrightarrow\quad
    V(S;c)\leq V(T;c).
    \label{eq:monotone}
\end{equation}

\subsection{A cost-regularised local index}

Let
\[
    B_\varepsilon(x)
    =
    \{j\in\Omega_t:\|z_j-x\|\leq\varepsilon\}
\]
denote the memories locally accessible from address $x$. A context $c$ is
mapped to a query position $\Phi(c)$, giving
\begin{equation}
    N_c(\Phi)
    =
    B_\varepsilon(\Phi(c)).
    \label{eq:neighbourhood}
\end{equation}

A naive objective would minimise only the information missing from this
neighbourhood,
\[
    V(\Omega_t;c)-V(N_c(\Phi);c).
\]
Under (\ref{eq:monotone}), however, this objective has a trivial solution:
mapping every context and memory to the same point puts all of $\Omega_t$
inside every ball and gives zero information loss.

We therefore charge a retrieval cost $\lambda>0$ per accessible memory:
\begin{equation}
    \boxed{
    \mathcal L_\lambda(\Phi)
    =
    \mathbb E\!\left[
        V(\Omega_t;C_t)
        -
        V(N_{C_t}(\Phi);C_t)
        +
        \lambda |N_{C_t}(\Phi)|
    \right].
    }
    \label{eq:indexloss}
\end{equation}
The outer expectation ranges over contexts and realised memory stores.
The constant embedding now pays $\lambda|\Omega_t|$ at every query and is no
longer universally optimal.

For a fixed context, minimising (\ref{eq:indexloss}) is equivalent to selecting
a locally accessible set with large
\begin{equation}
    J_\lambda(S;c)
    =
    V(S;c)-\lambda|S|.
    \label{eq:penalizedvalue}
\end{equation}
Thus $\lambda$ replaces an arbitrary top-$k$ budget by a value threshold:
another memory is worth making accessible precisely when the information it
adds justifies its retrieval cost.

\subsection{Complementarity and redundancy induce opposite pressures}

Define the marginal value of memory $j$ after already having access to $S$ as
\begin{equation}
    \Delta_j(S;c)
    =
    V(S\cup\{j\};c)-V(S;c).
    \label{eq:marginal}
\end{equation}
For two memories $i,j\notin S$, define their conditional interaction
\begin{align}
    s_{ij}(S;c)
    &=
    V(S\cup\{i,j\};c)
    -V(S\cup\{i\};c) \nonumber\\
    &\quad
    -V(S\cup\{j\};c)
    +V(S;c)
    \label{eq:interaction}\\
    &=
    \Delta_j(S\cup\{i\};c)-\Delta_j(S;c).
    \nonumber
\end{align}
Positive $s_{ij}$ is complementarity, meaning that observing $i$ increases the marginal
value of $j$. Negative $s_{ij}$ is redundancy.

\begin{observation}[Cost-adjusted co-retrieval]
\label{prop:cost}
For any context $c$, background set $S$, and $i,j\notin S$,
\begin{equation}
\begin{split}
    &J_\lambda(S\cup\{i,j\};c)
    -
    J_\lambda(S\cup\{i\};c)\\
    &\hspace{2cm}
    =
    \Delta_j(S\cup\{i\};c)-\lambda.
\end{split}
\label{eq:costmarginal}
\end{equation}
Hence joint retrieval is preferred to retrieving $i$ alone exactly when
$\Delta_j(S\cup\{i\};c)>\lambda$.
\end{observation}

\begin{proof}
Substitute (\ref{eq:penalizedvalue}) and cancel common terms.
\end{proof}

The observation gives the geometry both an attractive and an exclusion
criterion. If $i$ makes $j$ sufficiently valuable that
\[
    \Delta_j(S\cup\{i\};c)>\lambda,
\]
then a local neighbourhood benefits from making the pair jointly accessible.
Positive interaction $s_{ij}$ increases the contexts in which this occurs.

Conversely, suppose $j$ is redundant once $i$ is known. In the exact case,
\[
    V(S\cup\{i,j\};c)
    =
    V(S\cup\{i\};c),
\]
so $\Delta_j(S\cup\{i\};c)=0$ and retrieving both loses exactly $\lambda$
relative to retrieving only $i$. The index therefore has no incentive to place
both in every query neighborhood based on their high singleton values alone. Due to this exclusion pressure, a learned geometry should make
redundant alternatives separately reachable when possible. In particular,
$\|z_i-z_j\|>2\varepsilon$ is sufficient to ensure that no single
$\varepsilon$-ball contains both, although the objective does not require this
specific geometric solution.

This argument concerns decision-theoretic redundancy, instead of raw statistical
correlation between $M_i$ and $M_j$. Covariance of hidden representations is
representation-dependent and need not indicate whether two memories provide
the same information for the downstream decision.

\paragraph{Substitutes compete for the same coalitions.}
The interaction structure also formalises a stronger notion of redundancy.
Call memories $i$ and $j$ value-equivalent substitutes if, for every
context $c$ and every background set $S$ not containing them,
\begin{equation}
    V(S\cup\{i\};c)
    =
    V(S\cup\{j\};c)
    =
    V(S\cup\{i,j\};c).
    \label{eq:substitute}
\end{equation}
Then
\[
    \Delta_i(S;c)=\Delta_j(S;c)
\]
for every $S,c$. For another coalition $A$, define the change in $i$'s
marginal value caused by $A$,
\begin{equation}
    \eta_i(A\mid S,c)
    =
    \Delta_i(S\cup A;c)-\Delta_i(S;c).
    \label{eq:profile}
\end{equation}
Equation (\ref{eq:substitute}) immediately implies
\begin{equation}
    \boxed{
    \eta_i(A\mid S,c)=\eta_j(A\mid S,c)
    }
    \qquad
    \text{for all } A,S,c.
    \label{eq:sameprofile}
\end{equation}
Thus exact substitutes have the same individual value and additionally share same interaction profile with every other coalition. 
In this sense, redundant memories compete for the same synergistic sets.

This suggests that the learned address geometry should be driven by
value and interaction profiles, rather than by similarity of the raw
memory contents.

\subsection{Persistent retrieval as traversal}

Once a useful geometry exists, retrieval need not restart globally at every
step. Let $i_t\in\mathbb R^R$ be a persistent retrieval-head position and
define
\begin{equation}
    M_t
    =
    B_\varepsilon(\Phi(c_t))
    \cup
    B_\varepsilon(i_t),
    \qquad
    W_t
    =
    \bigcup_{\tau=t-H+1}^{t} M_\tau .
    \label{eq:window}
\end{equation}
The working window therefore retains recently traversed memories, while the
head continues to move locally.
The movement rule is left modular. It may be learned from
RL as in our experiment, moving to the memory of highest value in the last forwards pass, or by extending the existing key, value estimators. 
It is future work to find the movement rule which is most canonical given the construction of $\Phi$.

The regime in which persistence can help can be described without choosing a
fixed $k$. Define the cost-optimal retrieval set
\begin{equation}
    S_t^\star
    \in
    \argmax_{S\subseteq\Omega_t}
    \left[
        V(S;c_t)-\lambda|S|
    \right]
\end{equation}
and its one-step addition drift
\begin{equation}
    D_t
    =
    |S_t^\star\setminus S_{t-1}^\star|.
    \label{eq:drift}
\end{equation}
When $D_t$ is small relative to $|S_t^\star|$, most useful information persists
between consecutive retrievals. A persistent mechanism can in principle reuse
the previous retrieval state instead of rediscovering the full set.

\section{A highly crude fixed-$\Phi$ proof of principle}\label{s:method}

Learning $\Phi$ is the central unresolved part of the proposal. We therefore
test a deliberately weaker question. Given that the correct address space is supplied
by the environment, is then relative and persistent retrieval easier to learn than
selecting absolute memory locations?

\paragraph{A deliberately crude approximation to $\Phi$.}
The theoretical model learns $z_j=\Phi(C_j^{\mathrm w})$ from value relations.
The present experiment does not. Instead, it replaces $\Phi$ by a hand-designed
map derived from the spatiotemporal state at which the memory was created,
schematically
\begin{equation}
    \Phi_0(C_j^{\mathrm w})=(x_j,y_j,t_j).
    \label{eq:phi0}
\end{equation}
Thus the experiment assumes that physical proximity and temporal proximity are
useful indexing features.

\paragraph{Task and architecture.}
We use MiniGrid Memory \citep{minigrid2023}. The agent observes a cue, traverses
a corridor, and later chooses the corresponding object at a T-junction.
The policy uses a Transformer-XL backbone \citep{dai2019} and PPO
\citep{schulman2017}. At each step the action is
\[
    a_t=(a_t^{\mathrm{ext}},a_t^{\mathrm{mem}}),
\]
where the first component acts in the environment and the second moves the
retrieval head.

The memory address space is fixed to the environment's spatial and temporal
coordinates. In this early stage experiment no value-indexed embedding is learned. Retrieved observations and
current observations enter a shared FIFO context window. Models are first
pretrained without a retrieval head on the fixed-size task with context length
$64$, then fine-tuned on the randomised variant with context length $5$. In the
randomised task, the relevant cue lies between $7$ and $17$ environment steps
before the decision.

We compare relative head movement with two heads that select an absolute memory
position, parameterised respectively by a Gaussian and by a categorical
distribution. Six seeds are used for each retrieval-head condition and three
for the reduced-context no-head baseline.

\paragraph{Results.}
The primary learning-time statistic is the training step at which the
mean learning curve across seeds first exceeds $0.95$ return. We did not choose
the mean of per-seed threshold-crossing times, since some individual runs never
cross the threshold within the training budget.

\begin{table}[ht]
\centering
\caption{Comparison of retrieval heads on MiniGrid-Memory-S17Random (6 seeds each). Models were pretrained on MiniGrid-Memory-S9 without a retrieval head using a context size of 64. "Steps to 0.95" is the first training step at which the mean (min--max) return reaches 0.95. "> 2.0" indicates that the value was not determined after 2,048,000 steps. The entry \textit{No retrieval head (baseline)} was tested with only 3 seeds. }
\label{tab:retrieval_heads}
\begin{tabular}{lccc}
\toprule
Method & Steps to 0.95 (M) & Avg. Excess Steps & Context \\
\midrule
No retrieval head (as pretrained) & 0.14 (0.12--0.16) & 5.08 (4.80 -- 6.23) & 64 \\
\midrule
No retrieval head (baseline) & 1.38 (1.29 -- 1.43) & 15.0 (9.16 -- 20.73) & 5 \\
\textbf{Memory Navigation} & 0.59 (0.52 -- 0.65) & 6.18 (5.95 -- 6.33) & 5 \\
Absolute Position -- Gaussian & 0.94 (0.82 -- >2.0) & 6.32 (6.09 -- 6.58) & 5 \\
Absolute Position -- Categorical & >2.0 (0.63 -- >2.0) & 7.42 (5.81 -- 9.40) & 5 \\
\bottomrule
\end{tabular}
\end{table}

Relative movement reaches the $0.95$ threshold after $0.59$M steps, compared
with $0.94$M for Gaussian absolute addressing and more than $2.0$M for
categorical absolute addressing. It also produces trajectories close to the
full-context reference despite receiving only five context tokens.

This result supports only the preliminary claim, that incremental relative addressing is
easier to learn than the tested absolute-addressing alternatives in this
single-retrieval task. It does not test the learned indexing objective
(\ref{eq:indexloss}), interaction-driven geometry, or the low-drift hypothesis.

\paragraph{A revealing negative control.}
The reduced-context baseline is stronger than expected. With a short window,
the agent can encode the cue indirectly into its own behaviour. With temporal
encoding available, it waits for a cue-dependent number of steps and later
reads the resulting clock state. When temporal and spatial encodings are
removed and the window is shortened further, it can instead repeatedly turn in
the corridor and encode the cue in its recent trajectory. These are legitimate
solutions. They prevent us from claiming a capability gap and make the current
experiment evidence only for learning and behavioural efficiency.

\section{Limitations}\label{s:limit}

The strongest claims of the framework remain untested. The experiment
uses an environment-provided geometry rather than learning $\Phi$ from
(\ref{eq:indexloss}). Additionally, the experiment requires the retrieval of only one piece of information and therefore does not test whether the proposed method helps with sequentially accessing memories when they become useful.

It has also not yet been explored how well a fixed-dimensional geometry can represent the retrieval sets favored by $\mathcal{L}_\lambda(\Phi)$. If two memories $M_i$ and $M_j$ are quite similar, they may be redundant in many contexts, i.e.\ $s_{ij}(c) < 0$, while at the same time being complementary in many other contexts, i.e.\ $s_{ij}(c) > 0$. Can the optimal $\Phi$ position $z_i$ and $z_j$ such that they are contained in the same $\epsilon$-ball primarily in contexts in which they are complementary, or is such a spatial arrangement impossible?

Additionally, the retrieval rule is left modular rather than canonical. A slight modification to the proposed method is needed so that the movement of the retrieval head and the construction of $\Phi$ arise from the same underlying principle.

Finally, whilst the loss term is given, it is not described how $\Phi$ should be trained in clear detail. Exploring the behavior of $s_{ij}$ with respect to $\Phi$ may result in a class of additional constraints such that the loss can be decomposed into multiple terms that are easier to train. The training procedure itself is also not discussed.

\section{Conclusion}

We propose treating long-context retrieval as persistent traversal through an
address space learned from retrieval utility. Separating memory content from
memory address yields a simple organizing principle. Geometry should make
valuable combinations locally accessible while charging for unnecessary
co-retrieval. The resulting objective naturally distinguishes complementarity
from redundancy and shows that exact substitutes share the same interaction
profile with other memory coalitions.
Our current MiniGrid experiment tests only the traversal mechanism under fixed
geometry. The next empirical tests would be learning the value-indexed space and evaluating repeated low-drift retrieval.

\bibliographystyle{plainnat} 
\bibliography{bibliography}

\appendix
\section{Open access to data and code}\label{s:open_access}
The experimental code is available at
\url{https://anonymous.4open.science/r/Retrieval-as-Traversal-008D}.
A script for reproducing the results reported in this paper is included.
The experiments require less than 3\,GB of VRAM, and a single training run takes about three hours on a consumer-grade desktop computer.

\section{Declaration of LLM Usage}\label{s:llm}
General-purpose LLMs, such as ChatGPT and Claude, were used to assist the development of the experimental code. They were also used as interactive research assistants during the development of the theoretical formulation, including to critique mathematical constructions, find counterexamples and suggest possible refinements. Additionally, they assisted writing and editing. All resulting definitions, arguments, derivations, and claims included in the paper were reviewed and verified by the authors, who take full responsibility for the content.


\end{document}